\documentclass[12pt]{article}
\usepackage{arxiv}
\usepackage[numbers,sort&compress]{natbib}
\usepackage{booktabs} 
\usepackage{amsmath,mathtools}
\usepackage{amssymb}
\usepackage{algorithmic, algorithm}
\usepackage{xspace,xcolor}
\usepackage{dsfont}

\usepackage{mathtools,layout,graphicx,framed,xpatch,multirow,tabularx,gensymb,amssymb,enumitem,float,bbm,algorithm,algorithmic,adjustbox,booktabs,amsthm}

\newcommand{\R}{{\mathds{R}}}

\newcommand{\indic}[1]{\mathds{1}_{#1}}

\newcommand{\E}[1]{\mathrm{E}(#1)}
\newcommand{\ea}{(1+1)~EA\xspace}
\newcommand{\rls}{RLS\xspace}

\newcommand{\xopt}{x_{\mathrm{opt}}}

\newcommand{\smin}{s_{\mathrm{min}}}
\newcommand{\smax}{s_{\mathrm{max}}}

\newcommand{\wmax}{w_{\mathrm{max}}}
\newtheorem{theorem}{Theorem}
\newtheorem{definition}{Definition}
\newtheorem{lemma}{Lemma}
\newcommand{\Hell}{\tilde{z}}
\newcommand{\OneMax}{\textsc{OneMax}\xspace}
\newcommand{\onemax}{\OneMax}

\newcommand{\ie}{i.\,e.\xspace}
\newcommand{\eg}{e.\,g.\xspace}

\newcommand{\card}[1]{|#1|}

\renewcommand{\epsilon}{\varepsilon}
\DeclareMathOperator{\Prob}{Pr}
\newenvironment{proofof}[1]%
{\begin{trivlist}\item\textbf{Proof of #1.}}%
{\hspace*{\fill}\end{trivlist}}

\title{Improved Runtime Results for Simple Randomised Search Heuristics on Linear Functions with a Uniform Constraint}

\author{
Frank Neumann
\\Optimisation and Logistics
\\The University of Adelaide, Australia
\And
Mojgan Pourhassan
\\Optimisation and Logistics
\\The University of Adelaide, Australia\
\And
Carsten Witt
\\DTU Compute
\\Technical University of Denmark
}

\begin{document}

\maketitle

\begin{abstract}
In the  last decade remarkable progress has been made in development of suitable proof techniques for 
analysing randomised search heuristics. The theoretical investigation of these algorithms on classes of 
functions is essential to the understanding of the underlying stochastic process. Linear functions have 
been traditionally studied in this area resulting in tight bounds on the expected optimisation time 
of simple randomised search algorithms for this class of problems. Recently, the constrained version 
of this problem has gained attention and some theoretical results have also been obtained on this 
class of problems. In this paper we study the class of linear functions under uniform constraint 
and investigate the expected optimisation time of Randomised Local Search (RLS) and a simple evolutionary 
algorithm called \ea. We prove a tight bound of $\Theta(n^2)$ for RLS and improve the previously 
best known upper bound of \ea from $O(n^2 \log (Bw_{\max}))$ to $O(n^2\log B)$ in expectation and 
to $O(n^2 \log n)$  with high probability, where $w_{\max}$ and $B$ are the maximum weight of the 
linear objective function and the bound of the uniform constraint, respectively. Also, we obtain a tight bound of 
$O(n^2)$ for the \ea on a special class of instances. We complement our theoretical studies by experimental investigations 
that consider different values of $B$ and also higher mutation rates that reflect the fact that $2$-bit flips are crucial for dealing with the uniform constraint.
\end{abstract}


\section{Introduction}
Randomised search heuristics, such as evolutionary computing techniques and randomised local search 
algorithms have been widely used in real world applications that involve optimisation. Over the last  
decade a lot of progress has been obtained in understanding the runtime behaviour of 
these algorithms, which give us insights on the underlying stochastic process, particularly for classes of optimisation problems. 

One of the classes of problems which has been studied for a simple evolutionary algorithm, called \ea, is the 
 class of linear pseudo-boolean functions~\cite{djwea02, He2004, JJ08, Jagerskupper11, DJWLinearRevisited, WittCPC13}. 
The problem is to optimise a linear function of $n$ Boolean variables.  Droste, Jansen and Wegener \cite{djwea02}
 were the first to obtain
an upper bound of $O(n\log n)$ 
for the expected optimisation time of the \ea on this problem, where 
the presented proof is highly technical. Later, using the analytic framework of drift analysis~\cite{Hajek1982}, 
He and Yao presented a simplified proof for the same upper bound of $O(n \log n)$~\cite{ He2004}. Another 
major improvement was made in~\cite{JJ08, Jagerskupper11}, where the first precise analysis 
is presented for the optimisation time of the problem. Using a  framework for the analysis of multiplicative 
drift~\cite{DJWMultiplicativeDrift},  Doerr, Johannsen and Winzen  improved the precise upper bound result 
to the bound $(1.39 + o(1))en \ln n$~\cite{DJWLinearRevisited}.  \citet{WittCPC13} 
finally improved this bound  to $en \ln n + O(n)$, using adaptive 
drift~\cite{DoerrGoldbergLinear, DoerrGoldbergAdaptive} based on a novel potential function.

The mentioned results consider the problem without any constraints. However, the class of 
linear pseudo-boolean functions has also been recently studied under linear constraints~\cite{FRIEDRICH2018}.  
The problem of optimising a linear function under a linear constraint means that the search 
space is split by a hyperplane and only the points in one of the half spaces are considered 
feasible. This problem is equivalent to the well-known knapsack problem in the Boolean domain.
One of the linear constraints that is studied in~\cite{FRIEDRICH2018}, is the uniform constraint, 
in which the constraint is given by OneMax; hence, restricting the number of $1$-bits in the string. 
Denoting the bound on the number of $1$-bits by $B$, the authors of that work have conjectured a 
general upper bound of $O(n^2)$ for all linear functions, independently of $B$. However, their analysis only proves
a general upper bound of $O(n^2 \log(Bw_{\max}))$  for this setting, where $w_{\max}$ is the largest weight in the objective function. 

In this paper, we study two randomised search heuristics, RLS and \ea, and 
analyse the expected optimisation time of these algorithms on linear functions problem under a 
uniform constraint. We first prove that an upper bound of $O(n^2)$ holds for RLS and then we improve 
the current upper bound of $O(n^2 \log(Bw_{\max}))$ to $O(n^2\log^+ B)$ for the \ea, where $\log^+(x)\coloneqq \min\{1,\log x\}$. 
In the special case that the $B$ smallest weights of the linear function are identical, the bound for the \ea 
becomes $O(n^2)$. 
Together with the lower bound 
$\Omega(n^2)$ due to \cite{FRIEDRICH2018}, we have obtained asymptotically tight results.

The problem of optimising a linear function under a uniform constraint can be seen as 
a simplification of the classical minimum spanning tree problem. The minimum spanning tree problem 
has been studied quite extensively in the area of randomised search heuristics. Neumann 
and Wegener~\cite{NeumannWegenerTCS07} have shown an upper bound of $O(m^2(\log n + \log w_{\max}))$, 
where $n$ is the number of nodes, $m$ is the number of edges and $w_{\max}$ is the largest edge weight 
of the given input graph. These results have been improved for special classes of 
graphs~\cite{WittGECCO14} and edge weights~\cite{ReichelSkutellaFOGA09}. However, it still remains an open 
question whether an upper bound of $O(m^2 \log n)$ can be achieved for the \ea on any graph.

The investigations in this paper are on a simpler problem and do not have direct implications for instances 
of the minimum spanning tree problem, but we are hopeful that the provided techniques 
and insights will be helpful to achieve an upper bound of $O(m^2 \log n)$ of the \ea on that problem. 
Many other analyses of evolutionary algorithms also contain the largest input 
weight in the obtained runtime bounds (\eg, \cite{ReichelSkutella10,NeumannReichelSkutella11}) and getting strong results independent of this parameter 
poses a significant challenge for many problems where input weights might be exponential. This includes many 
results using multiplicative drift analysis when dealing with exponentially large weights and using the given fitness 
functions as the potential/drift function~\cite{DJWMultiplicativeAlgorithmica}.

This paper is structured as follows. Section~\ref{sec:prel} includes the definition of the 
investigated algorithms  and the analytical tools that we are going to use in the paper. 
Section~\ref{sec:scenario} explains the studied problem, as well as the notations that we use in this paper. In 
Sections~\ref{sec:RLS} and~\ref{sec:ea}, respectively, we present the analysis for RLS and the \ea. We report on our experimental results in Section~\ref{sec:experiments} and finish in Section~\ref{sec:concl} with some conclusions.

\paragraph{Extensions to the conference version:} The conference version \cite{NPWGECCO19} of this paper 
only proved a runtime bound $O(n^2\log B)$ with respect to the \ea. For a special class of functions where the 
$B$ smallest weights are identical we improved this bound 
towards the asymptotically tight result $O(n^2)$ via a refined drift analysis.
 Moreover, we have the new Section~\ref{sec:experiments} describing 
empirical studies 
of the impact of the bound~$B$ and the mutation rate.

\section{Preliminaries}
\label{sec:prel}
We consider two classical randomised search heuristics called RLS and \ea, see   
Algorithms~\ref{alg:oneoneea} and~\ref{alg:rls}, which are intensively 
studied in the theory of randomised search heuristics \cite{AugerDoerrBook,JansenBook}. 
The \ea is a globally searching hill-climber, whereas RLS samples from a neighbourhood of size at most~$2$. Note that for RLS, steps 
that change two bits are crucial when the current search point has a tight constraint but is not the optimum yet. 

The runtime 
 (synonymously, optimisation time) of the algorithms is defined as
 the random number of iterations until an optimal search point has 
been sampled. Denoting this number by a random variable $T$, in this paper we analyse the expected value 
of~$T$, $\E{T}$, for both studied algorithms.

\begin{algorithm}[t]
\caption{\ea}
\label{alg:oneoneea}
\begin{algorithmic}
\STATE $t:=0$.
 \STATE Choose uniformly at random $x_0 \in \{0,1\}^n$.
 \REPEAT 
  \STATE Create $x'$ by flipping each bit in $x_t$ independently with probability $1/n$.
  \STATE $x_{t+1}:=x'$ if $f(x') \le f(x_t)$, and $x_{t+1}:=x_t$ otherwise. 
  \STATE $t:=t+1$.
 \UNTIL{some stopping criterion is fulfilled.}
\end{algorithmic}
\end{algorithm}

\begin{algorithm}[t]
\caption{Randomised Local Search (RLS)}
\label{alg:rls}
\begin{algorithmic}
\STATE $t:=0$.
 \STATE Choose uniformly at random $x_0 \in \{0,1\}^n$.
 \REPEAT 
  \STATE Choose $b\in\{1,2\}$ uniformly. Create $x'$ by flipping $b$ bits in~$x_t$ chosen uniformly at random.
  \STATE $x_{t+1}:=x'$ if $f(x') \le f(x_t)$, and $x_{t+1}:=x_t$ otherwise. 
  \STATE $t:=t+1$.
 \UNTIL{some stopping criterion is fulfilled.}
\end{algorithmic}
\end{algorithm}

In our analysis for the \ea, we use two important drift theorems, 
which we list in this section in Theorem~\ref{theo:variableDrift} and~\ref{theo:multdrift-upper}. 
The variable drift theorem (Theorem~\ref{theo:variableDrift})
was independently proposed in \cite{MitavskiyVariable,AdaptiveDriftJohannsenPhD2010} and generalised in 
\cite{AdaptiveDriftRowe2012}. 
The multiplicative drift theorem (Theorem~\ref{theo:multdrift-upper}) is due to 
\citet{DJWMultiplicativeAlgorithmica} and was enhanced with tail bounds by
  \citet{DoerrGoldbergAdaptive}. Both theorems are formulated in a unified 
and slightly generalised manner here. The formulation in 
terms of an arbitrary stochastic process 
can also be found in \cite{LehreWittISAAC2014}. The adaptation of the multiplicative drift theorem 
to arbitrary positive $\smin$-values 
has first been stated in~\cite{DJWMultiplicativeAlgorithmica}.

\begin{theorem} [Variable Drift, cf.~\cite{AdaptiveDriftRowe2012, AdaptiveDriftJohannsenPhD2010}]
\label{theo:variableDrift}
Let $(X_t)_{t\ge 0}$, be a stochastic process, adapted to a filtration~$\mathcal{F}_t$, 
over some state space $S\subseteq \{0\}\cup [\smin,\smax]$, where 
$\smin>0$.
Furthermore, let $T\coloneqq\min\{t\mid X_t=0\}$ be the 
first hitting time of state~$0$. 
Suppose that there exists a monotonically increasing 
function $h\colon [\smin, \smax] \rightarrow \R^+$ such that $1/h$ is integrable, and for all $t<T$
\[\E{X_t-X_{t+1}\mid \mathcal{F}_t}\geq h(X_t)\text{.}\]
Then, 
\[\E{T\mid \mathcal{F}_0}\leq \frac{\smin}{h(\smin)}+\int_{\smin}^{X_0} \frac{1}{h(s)} \text{d}s.\]
\end{theorem}

\begin{theorem} [Multiplicative Drift, cf. \cite{DJWMultiplicativeAlgorithmica,DoerrGoldbergAdaptive}] 
\label{theo:multdrift-upper}
Let $(X_t)_{t\ge 0}$, be a stochastic process, adapted to a filtration~$\mathcal{F}_t$, 
over some state space $S\subseteq \{0\}\cup [\smin,\smax]$, where 
$\smin>0$. Suppose that there exists a $\delta>0$ such that
for all $t\ge 0$
\[
\E{X_t-X_{t+1}\mid \mathcal{F}_t}\ge \delta X_t.
\]
\parbox{\textwidth}{%
Then it holds for the first hitting time 
$T:=\min\{t\mid X_t=0\}$ that 
\[
\E{T\mid \mathcal{F}_0} \le \frac{\ln(X_0/\smin)+1}{\delta}.
\]
Moreover, 
$\Prob(T> (\ln(X_0/\smin)+r)/\delta) \le e^{-r} $ for any $r>0$.}
\end{theorem}

Finally, in our analysis, we will use the following simple lemma dealing with convexity.

\begin{lemma}
\label{lem:convexity}
Let $a_1,\dots,a_k\ge 0$ and $C>1$. Then 
\[
(a_1+\dots + a_k)^{C} \le k^{C-1}\left((a_1)^{C} + \dots + (a_k)^{C}\right).
\]
\end{lemma}
\begin{proof}
We write
\[
(a_1+\dots + a_k)^{C} = k^C \left(\frac{a_1}{k}+\dots+\frac{a_k}{k}\right)^C,
\]
and interpret the expression in parentheses as a linear combination of the $k$ numbers 
with coefficient $1/k$ each. Applying Jensen's inequality, we have 
\[
\left(\frac{a_1}{k}+\dots+\frac{a_k}{k}\right)^C \le \frac{(a_1)^C}{k} + \dots + \frac{(a_k)^C}{k},
\]
which, after multiplying with $k^C$, gives the desired result.
\qed\end{proof}

\paragraph{Notation} Throughout this paper, for natural numbers $n$ we write 
$[n]\coloneqq \{1,\dots,n\}$. 

\section{Scenario}
\label{sec:scenario}
In this paper we analyse the expected optimisation time of RLS and the \ea and consider an optimisation problem 
with a linear objective function under a uniform constraint. In contrast to earlier work in this area 
\cite{FRIEDRICH2018}, we assume that the objective function has to be \emph{minimised} since this perspective
 more naturally coincides with the minimisation of the distance to the target~$0$ that is implicit in the drift 
theorems (Theorems~\ref{theo:variableDrift} and~\ref{theo:multdrift-upper}). The upper bounds on the 
optimisation time obtained for RLS in Section~\ref{sec:RLS} and for the \ea in Section~\ref{sec:ea}, respectively, 
would equally hold if we adopted maximisation in the same way as in the previous work. See Section~\ref{sec:min-vs-max} for more discussion about assuming a minimisation problem or a maximisation problem.

Formally, throughout this paper, 
we consider the search space $\{0,1\}^n$ of all bit strings $x=x_n x_{n-1}\cdots x_1$, and the goal is to \emph{minimise} the objective function of
\[
f_{\mathrm{obj}} (x)=\sum_{i=1}^n w_i x_i,
\]
where $w_n\ge \dots\ge w_1$ are positive real weights, 
under the uniform constraint 
\[x_1+\dots+x_n\ge B\]
for some $B\in\{1,\dots,n\}$. We are excluding $B=0$, as it is equivalent to having no constraints.
Note that we follow common conventions in the analysis of linear functions 
\cite{DJWLinearRevisited, WittCPC13} by writing down 
search points in the order $x_n\dots x_1$, \ie, most significant bit first. 
Therefore, an index $i$ is called \emph{left} of index~$j\neq i$ 
if $i>j$ and \emph{right} of~$j$ otherwise.

A search point is \emph{optimal} if it minimises $f_{\mathrm{obj}}$ and is placed 
in the feasible region, \ie, the part of the search space where the constraint is satisfied. 
Moreover, we say that a search point is \emph{tight} (in the constraint) if
$x_1+\dots+x_n = B$.

In Algorithms~\ref{alg:oneoneea} and~\ref{alg:rls}, $x$ denotes the best search point 
found so far, and $x'$ is the new offspring, which replaces $x$  if it is at least 
as good as it with respect to a fitness function $f$ that we define as follows.  We 
aim to handle the  constraint of the problem  by setting a penalty for the violation. Therefore, we define the fitness function below, to be used by the algorithms.
\[f(x)=f_{\mathrm{obj}}(x)+ \max\{0, (B-b(x))\}\cdot (n\wmax +1),\]
where  $\wmax = w_n$ is the maximal weight, and $b(x)=\sum_{i=1}^n x_i$ 
is the number of ones in the bit string $x$, which we also refer to as the $b$\nobreakdash-value of $x$. For feasible search points we have 
$b(x)\geq B$, which implies that $\max\{0, (B-b(x))\}=0$. Therefore, the 
penalty of $(B-b(x))\cdot (n\wmax +1)$ is applied to search points that are infeasible, 
making the value of this fitness function larger than that of any feasible search points. 
Note that with this definition of the fitness function, the search in infeasible region is also guided to the 
feasible region, because as the extent of the constraint violation is reduced the penalty also decreases.

We first find a tight bound on the expected optimisation time of RLS on this problem in Section~\ref{sec:RLS}, 
and then focus on the challenging analysis of the \ea in the rest of the paper. Lemma~\ref{lem:RLSFeasible}, which 
is presented in Section~\ref{sec:RLS} holds for the \ea as well as RLS, 
and is used in the analysis of both algorithms (Section~\ref{sec:RLS} and Section~\ref{sec:ea}).

\section{Analysis of RLS}
\label{sec:RLS}
In Theorem~\ref{theo:RLS}, we prove that RLS (Algorithm~\ref{alg:rls}) optimises 
the linear function problem under a uniform constraint in expected time $O(n^2)$. In Theorem~\ref{theo:RLSLowerBound} we also prove that this bound is tight. 

We start with the following lemma, which proves that a feasible search point is sampled by the studied algorithm in 
$O(n\log (n/(n-B)))$. This lemma holds for the \ea as well, and is also used 
in the analysis of Section~\ref{sec:ea}. The proof of this lemma is similar to the 
proof of Lemma~7 in~\cite{FRIEDRICH2018} in which a maximisation problem for a linear function 
under uniform constraint is considered. Here we  adapt the proof to match the minimisation problem.

\begin{lemma}
\label{lem:RLSFeasible}
Starting with an arbitrary initial solution, the expected time until RLS or the \ea obtain a feasible solution is $O(n \log (n/(n-B)))$.
\end{lemma}
\begin{proof}
Recall that we denote by $b(x)$ the number of 1-bits in a search point $x$. Due to the 
definition of the fitness function $f$, in the infeasible region, a search 
point $x$ with a larger $b(x)$ is always preferred to a 
search point with a smaller $b$\nobreakdash-value. Therefore, the problem can be 
seen as maximising $b(x)$ until reaching $b(x)\geq B$, where the 
initial solution may have a $b$\nobreakdash-value of 0. We consider the potential function 
\begin{equation*}
  g(x)=\begin{cases}
    n-b(x), & \text{if $b(x)<B$},\\
    0, & \text{otherwise},
  \end{cases}
\end{equation*}
for which the initial value is at most $n$ and the minimum value before 
reaching 0 is $n-B+1$. The value of this function is never increased during the process of RLS 
or the \ea as larger $b$\nobreakdash-values are always preferred to smaller $b$\nobreakdash-values 
before reaching $g(x)=0$.  We find the drift on the value of $g(x_t)$ for RLS, where $x_t$ 
is the search point of the algorithm at step $t$, as
\[\E{g(x_{t})-g(x_{t+1})\mid g(x_t);g(x_t)>0}\geq \frac{n-b(x_t)}{2n}= \frac{g(x_{t})}{2n}\]
since RLS performs a 1-bit flip with probability $1/2$ and flips a 0-bit with 
probability $(n-b(x_t))/n$, improving $g$ by 1. A similar drift of 
\[\E{g(x_{t})-g(x_{t+1})\mid g(x_t);g(x_t)>0}\geq  \frac{g(x_{t})}{en}\]
is obtained for the \ea, in which the probability of flipping one 0-bit and no other bits is 
$\frac{n-b(x_t)}{n}\cdot (1-\frac{1}{n})^{n-1}\geq \frac{g(x_{t})}{en}$.

Using the multiplicative drift theorem (Theorem~\ref{theo:multdrift-upper}) with  
$\delta= 1/en$, $X_0\leq n$ and $\smin= n-B+1$ we find that the expected time until reaching 
a feasible solution is upper bounded by 
\[\frac{\ln(n/(n-B+1))+1}{1/(en)} = O\left(n\log \left(\frac{n}{n-B}\right)\right).\]
\qed
\end{proof}

\begin{theorem}
\label{theo:RLS}
Starting with an arbitrary initial solution, the expected optimisation time of \rls on linear functions with a uniform constraint is $O(n^2)$.
\end{theorem}

\begin{proof}
Due to Lemma~\ref{lem:RLSFeasible}, RLS finds a feasible solution in expected 
time $O(n\log (n/(n-B)))$. Also, since all feasible solutions have smaller fitness values than 
infeasible solutions, the algorithm does not switch back to the infeasible region again. Moreover, 
note that once a feasible solution has been found for the first time, the number 
of ones in the solution cannot be increased. This is due to the fact that the penalty is $0$ 
and all $1$-bit flips flipping a $0$ increase the fitness. Also, all $2$-bit flips 
that increase the number of ones (flipping two zeros) increase the fitness as well.

We split the analysis of the algorithm after obtaining a feasible solution into two phases. 
In the first phase, the algorithm starts with a solution $x$ with $b(x)>B$ and obtains 
a solution with exactly $B$ $1$-bits ($b(x)=B$). Then the second phase starts, during which the number 
of $1$-bits of the solution does not change, because  both $1$-bit flips and $2$-bit flips that 
change the number of ones increase the fitness. If the first feasible solution that is obtained 
by the algorithm has $b(x)=B$, then we do not have a first phase.
We first analyse the expected time until the first phase ends, then we focus on the second phase.
 
In the first phase, the algorithm starts with a solution $x$ with $b(x)>B$. In this 
situation, as explained above, $b(x)$ does not increase. Moreover, a $1$-bit flip that flips a~$1$ to~$0$, 
which happens with probability $b(x)/(2n)$, is always accepted because it decreases the fitness, while 
not violating the constraint yet. By defining a potential function $g(x)$ as 
\begin{equation*}
  g(x)=\begin{cases}
    b(x), & \text{if $b(x)>B$},\\
    0, & \text{otherwise},
  \end{cases}
\end{equation*}
and using the multiplicative drift theorem with $\delta= 1/2n$, $X_0\leq n$ and $\smin= B+1$, we 
find the expected time of $O(n \log (n/B))$ until a solution with $g(x)=0$ is found, which implies $b(x)=B$. 
 
Now we analyse the second phase. Having obtained a solution with exactly $B$ ones, 
only $2$-bit flips flipping a zero and a one are accepted.
Let $r$ be the number of bits of weight $w_B$ among $w_B, \ldots, w_1$, \ie, $r=|\{i \mid w_i=w_B, 1 \leq i \leq B\}|$.
An optimal solution contains all weights of weight less than $w_B$ and exactly $r$ weights of weight $w_B$.

Let $x$ be the current solution and \[
s(x)=\max \{0, r-|\{i \mid w_i=w_B \wedge x_i=1\}|\}\]
 be the number of $1$-bits of weight $w_B$ missing in $x$. Furthermore, let 
\[t(x) = |\{i \mid w_i<w_B \wedge x_i=0\}|
\]
be the number of $1$-bits of weight less than $w_B$ missing in $x$.

We denote by 
\[k= s(x) + t(x)
\]
the number of weights that are missing in the weight profile of the current solution $x$ compared to an arbitrary optimal solution.



As there are exactly $B$ $1$-bits in the current solution $x$, it implies that there are exactly
\[
k= \{i \mid w_i>w_B \wedge x_i=1\} + \max\{0, |\{i \mid w_i=w_B \wedge x_i=1\}|-r\}
\]
weights chosen in $x$ that do not belong to an optimal weight profile. Note that for a given solution $x$,
\[r-|\{i \mid w_i=w_B \wedge x_i=1\}|
\]
is a fixed value which is greater than $0$ if $1$-bits of weight $w_B$ are missing and less than $0$ if there are too many $1$-bits of weight $w_B$.

This implies that there are at $k$ $1$-bits which can be swapped with an arbitrary $0$-bit of the  missing $k$ weights in order to reduce $k$. 
Hence, the probability of swapping a 1-bit with a 0-bit of the missing weights is at least
$\frac{k^2}{2n^2}$ and the expected waiting time for this event is bounded from above 
by $2n^2/k^2$. Since $k$ cannot increase, it suffices to sum up these expected waiting times 
following the idea of 
 fitness-based partitions \cite{WegenerMethods}. Hence, 
the expected time until reaching $k=0$  is
\[\sum_{k=1}^B (2n^2/k^2) = O(n^2),\]
which completes the proof.
\qed\end{proof}

We now show that the previous bound is asymptotically tight.

\begin{theorem}
\label{theo:RLSLowerBound}
There is a linear function $f$ and a bound $B$ such that, starting with a uniformly random initial solution, the 
expected optimisation time of RLS on $f$ under uniform constraint $B$ is $\Omega(n^2)$.
\end{theorem}

\begin{proof}
The same lower bound is proved for the \ea in Theorem~10 of~\cite{FRIEDRICH2018}. Since RLS does not flip more than $2$~bits at 
each step, the proof of this theorem is simpler. We use a function $f$ that is similar to the function that is used in~\cite{FRIEDRICH2018} and is adapted for a minimisation problem. We define $f$ as
\[f(x)= \sum_{i=1}^{B} x_i+ \sum_{j=B+1}^n (1+\epsilon)x_j,\]
where $\epsilon$ is an arbitrary positive real number.
Since the weights that are assigned to the first $B$ bits are smaller than the weights of other bits, the optimal solution is selecting the first $B$ bits. We prove that with $B=n/4$, the expected optimisation time of RLS is lower bounded by $\Omega(n^2)$.

We denote the Hamming distance of a solution $x$ to the optimal solution by $d_H(x)$. By Chernoff bounds the initial solution has at least $n/3$ 1-bits 
with probability exponentially close to 1, which implies a Hamming distance of at least $n/12$ to the optimal solution. Since RLS can only decrease the Hamming distance 
by one or two at each step, in order to reach the optimal solution, a solution $x$ has to be obtained at some point such that $2\leq d_H(x)\leq 3$. We investigate the 
process based on the number of 1-bits of solution $x$, which we denote by $|x|_1$. Since the initial solution is feasible with probability exponentially 
close to~$1$, we either have $|x|_1=B$ or $|x|_1>B$.

If $|x|_1=B$,  then $d_H(x)=2$ and $x$ can only have one 0-bit among the first $B$ bits and 
one 1-bit among other bits. In this case only a swap on the two misplaced bits can improve the 
fitness, the probability of which to happen is at most $1/n^2$; hence, the waiting time is $\Omega(n^2)$ and the 
theorem follows.

If $|x|_1>B$, then flipping any of the 1-bits improves the fitness. Since there are more than $n/4$ 1-bits in the solution, 
the probability of decreasing the number of 1-bits is at least $1/8$ at each time step of RLS. Furthermore, the number of 
0-bits does not decrease by RLS due to the fitness function. Using a drift argument on $|x|_1-B$, we find that in 
expected constant time (at most $\frac{3}{1/8}$) a solution $x'$ is obtained such that $|x'|_1=B$. This implies 
that in a phase of $\log n$ steps, with probability $1-o(1)$ the solution $x'$ is obtained. 
If $x'$ is not optimal, 
then we have to swap at least two bits and the theorem follows as above. 
What remains is to show that $x'$ is not optimal with 
probability~$1-o(1)$.  
Since $d_H(x)\leq 3$, the probability of flipping a one-bit from~$x$ that is outside the first $B$ positions is at most 
$3/n$ at each step. Therefore, with probability at least $1-(1-(1-1/n)^{\log n})^3 = 1-o(1)$ 
at least one of these bits does not flip  in a 
phase of $\log n$ steps; hence, $x'$ is not the optimal solution with probability~$1-o(1)$, which completes the proof.
\qed\end{proof}

\section{Analysis of the \ea}
\label{sec:ea}
In this section we analyse the expected optimisation time of the \ea for 
linear functions under a uniform constraint. In the following subsection we present 
the statement of our results, and in the subsequent section we prove the statement.

For a linear function under uniform constraint of $B$, we aim to prove that the \ea finds an optimal 
solution in expected time $O(n^2 )$. 
Since Lemma~\ref{lem:RLSFeasible} proves that a feasible solution is obtained by 
the \ea in expected time $O(n\log(n/(n-B)))$ and this upper bound is asymptotically smaller than $O(n^2)$, 
we only focus on the analysis of the algorithm after finding a feasible solution. 
The main theorem that we prove in this section is stated below.

\label{sec:mainTheorem}
\begin{theorem}
\label{theo:upper}
Given an arbitrary linear function 
under a uniform constraint $x_1+\dots+x_n\ge B$ for $B\in\{1,\dots,n\}$,
the expected optimisation time of the  \ea 
is upper bounded by $O(n^2\log^+ B)$. Also, the time is  $O(n^2 \log n)$ with probability $1-O(n^{-c})$ 
for any constant~$c>0$.
\end{theorem}

To prove Theorem~\ref{theo:upper}, we conduct an adaptive drift analysis, where 
the underlying potential function $g(x)$, to be minimised, depends on both the weights $(w_1,\dots,w_n)$ of 
the linear function 
 and the constraint value~$B$. 
The exact definition of the potential function is to some extent 
 inspired by the techniques developed in \cite{WittCPC13} and
further applied in \cite{DoerrPohlGECCO12} and \cite{DSWFOGA13}. However, 
as these papers are concerned with unconstrained problems only, 
additional effort has been made to transfer these techniques to our scenario.  

Once having defined the potential function, the idea 
is to analyse the potential $X_t:=g(x^{(t)})$ of the random search 
point~$x^{(t)}$ maintained
by the \ea on~$f$ at time~$t$.  We 
bound  its expected change $\E{X_t-X_{t+1} \mid X_t}$, \ie, the 
expected decrease of the potential function from time~$t$ to time~$t+1$. Then we 
use this bound  in the drift argument that proves the main theorem.

The following lemma (Lemma~\ref{lem:drift-statement}) states this bound
as well as a bound on the maximum value of the potential function, which will be required
in the drift theorems. We define $g(x)$ later in Definition~\ref{def:potential}, and prove 
the statements of Lemma~\ref{lem:drift-statement} for this function afterwards. We 
first bring the statement of this lemma and show how it can be used to prove Theorem~\ref{theo:upper}.

\begin{lemma}
\label{lem:drift-statement}
Considering a random variable $X_t= g(x^{(t)})$, where the function $g$ is given 
in Definition~\ref{def:potential} and $x^{(t)}$ is the random search point of the \ea at time $t$, for all time steps $t$ we have
\begin{enumerate}
\item
$\E{X_t-X_{t+1} \mid X_t} \ge \frac{0.025}{en^2} \min\left\{\frac{X_t^{8/7}  }{B^{2/7}}, X_t\right\}$,
\item
$1 \le X_t = O(n^9) $ if $x^{(t)}$ is not optimal.
\end{enumerate}
\end{lemma}

Deferring the definition of the potential function~$g$ and the proof of  the previous lemma, we obtain our theorem.

\begin{proofof}{Theorem~\ref{theo:upper}}
We apply the variable drift theorem (Theorem~\ref{theo:variableDrift}) given 
the statements of Lemma~\ref{lem:drift-statement}. 
Using that $X_t\ge 1\eqqcolon \smin$ 
and $X_t = O(n^9)$ as well as the drift bound 
 \[h(X_t)\coloneqq \frac{ 0.025 X_t} { en^2} \max\{X_t^{1/7} / B^{2/7},1\}\text{,}\]
the expected optimisation time  is bounded by 
\begin{align*}
& \frac{\smin}{h(\smin)} + \int_{\smin}^{n^8} \frac{1}{h(x)} \,\mathrm{d}x \\ 
& = O(n^2) + \frac{en^2}{0.025}\left(\int_{1}^{B^2} \frac{1}{x} \mathrm{d}x+ 
B^{2/7}\int_{B^2+1}^{O(n^9)} \frac{1}{x^{8/7}} \mathrm{d}x \right)\\
& = O(n^2) + O(n^2)(O(\log B) + O(1)) = O(n^2\log^+ B),
\end{align*}
which completes the proof of the $O(n^2 \log^+ B)$ bound.

For the tail bound we use the multiplicative drift theorem (Theorem~\ref{theo:multdrift-upper}) 
with the simple bound $\E{X_t-X_{t+1} \mid X_t} \ge \frac{ 0.025 X_t} { en^2} $
of Lemma~\ref{lem:drift-statement} along with $X_t \le   n^8$ that implies $\ln(X_t/\smin)=O(\log n)$. 
Note that the theorem gives the  upper bound $O(n^2\log n)$ on the expected optimisation time so that  the tail bound  
can be obtained by setting $r=c\ln n$.
\qed\end{proofof}

In the following, we unroll the proofs of the drift statements. 
The proof of Lemma~\ref{lem:drift-statement} relies on the analysis of the drift of
the potential function~$g\colon \{0,1\}^m\to\R$.
We now introduce the setup required to define $g(x)$.

\begin{definition}
\label{def:potential}
Let an arbitrary linear function $f=\sum_{i=1}^n w_i x_i$, where $w_n\ge \dots\ge w_1$, 
under uniform constraint $x_1+\dots+x_n\ge B$ be given and let $\xopt$ be the (not necessarily unique) 
optimal search point having one-bits at the $B$ rightmost positions only. Let $m\coloneqq \card{\{w_{1},\dots,w_n\}}$ 
be the number of distinct weights and define  
$s(i)=\min\{j\mid \card{\{w_{1},\dots,w_j\}}\ge i\}$, where $i\in\{1,\dots,m\}$, 
as the start of the block of indices  having the $i$th largest weight
 as well as $s(m+1)\coloneqq n+1$. 
Also, let $K_i\coloneqq \{s(i),\dots,s(i+1)-1\}$ be the indices 
comprising the $i$th block of equal weights. 

For $j\in [n]$, we let 
\begin{equation*} \label{eqn:gamma_i}
\gamma_j \coloneqq 
\begin{cases}
1 & \text{ if $j\le B$,}\\
75B (j-B)^7  & \text {otherwise}.
\end{cases}
\end{equation*}
Based on this, define for all blocks $i\in \{2,\dots,m\}$ 
\begin{align*} \label{eqn:weights_g}
g_{s(i)} = \dots = g_{s(i+1)-1} & \coloneqq  \min\{\gamma_{s(i)}, g_{s(i-1)} \cdot w_{s(i)}/w_{s(i-1)}\}
\end{align*}
as well as $g(1)\coloneqq 1$.  
Finally, we let $g(x)\coloneqq \sum_{j=1}^n g_j x_j - \sum_{j=1}^B g_j $, \ie, $g(\xopt)=0$. 
To prepare the drift analysis, we define any block $i \in [m]$:
\begin{itemize}
\item $\kappa(i) := \max \{ j \le i \mid g_{s(j)} = \gamma_{s(j)} \}$, the most significant
block right of~$i$ (possibly $i$ itself) capping  according to the sequence $\gamma_i$,
\item $L(i) := \{ m,\dots,\kappa(i) \}$, the block indices left of (and including) the block $\kappa(i)$,
\item $R(i) := \{ \kappa(i) - 1,\dots,1\}$, the block indices right of block~$\kappa(i)$.
\end{itemize}
This concludes the definition of the potential function.
\end{definition}

We now work out some important properties of the potential~$g$ and along the way, present some underlying 
intuition for the definition. 
Considering the original weights $w_1,\dots,w_n$ in increasing order, 
the potential function assigns the same $g$-value to all indices within a block 
$K_i$ of equal $w$-value. Note that blocks may be of size~$1$. We also observe that 
the weights of $g$ can be equivalently defined as 
 $g_j=\min\{\gamma_{j}, g_{j-1} \cdot w_{j}/w_{j-1}\}$ for $j\in[n]$.

The idea of the potential function is
to cap the original weights at~$\gamma_i$ at the indices where the original weights increase too steeply 
and to rebuild their slope otherwise. In particular, we have $g_i\le \gamma_i$ for all $i\in[n]$. 
The intuition
is that the potential function will underestimate the progress made at blocks being at least as significant
as $\kappa(i)$, \ie, the blocks in $L(i)$. 	
In all less significant blocks (those in $R(i)$), we will pessimistically assume that they 
contribute a loss, and the choice of $\kappa(i)$ guarantees that this loss is overestimated. We 

As already mentioned, the potential function assigns~$0$ to all optimal search points (which are unique if and only if 
$w_B \neq w_{B+1}$). Inspecting the definition, we have 
\begin{align*}
g(x) & = \sum_{j > B\mid x_j=1} g_j    +\sum_{j \le B\mid x_j=1} g_j - \sum_{j=1}^B g_j \\
& = \sum_{j > B\mid x_j=1} g_j -  \sum_{j \le B\mid x_j=0} g_j,
\end{align*}
Hence, the accumulated weight of the one-bits outside the $B$ rightmost positions is an upper 
bound on the $g$-value; formally, 
\begin{equation}
g(x) \le \sum_{j > B\mid x_j=1} g_j .
\label{eq:g-respects-ones-and-zeros-1}
\end{equation}


As mentioned above, we will analyse the stochastic process $(X_t)_{t \ge 0}$
where $X_t = g(x^{(t)})$ for all $t$, and define $\Delta_t :=
X_t - X_{t+1}$. Recall that we are interested in the first point in time~$t$ where 
$X_t=0$ holds since $g(x)=0$ if and only if $f(x)=f(\xopt)$. The drift $\E{\Delta_t\mid X_t}$ of the potential function will be worked
out conditioned on certain events depending on two flipping bits. The following notions prepare the definition
of these events.

\begin{definition}
\label{def:idiuetc}
Given $x^{(t)}\in\{0,1\}^n$, denote by $x'$ the random search 
point created by mutation of $x^{(t)}$ (before selection). We define
\begin{itemize}
\item $I := \{ i \in [n] \mid x_i^{(t)}=1\}$ the one-bits in~$x^{(t)}$,
\item $I^* := \{i \in I \mid  x'_i=0 \}$ the one-bits flipping to~$0$,
\item $Z := \{ i \in [n] \mid x_i^{(t)}=0 \}$ the zero-bits in~$x^{(t)}$,
\item $Z^* := \{i \in Z \mid x'_i=1  \}$ the zero-bits flipping to~$1$.
\item $\sigma_i \coloneqq \card{I^*\cap K_i} - \card{Z^*\cap K_i}$ the surplus 
of flipping one-bits within block $K_i$, where $i\in [m]$. 
\end{itemize}
\end{definition}
Note that the random sets $I^*$ and $Z^*$ are disjoint and 
 that the remaining bits in $[n]$ 
contribute nothing to the $\Delta_t$-value.

Obviously, for $\Delta_t\neq 0$ it is necessary that $x^{(t+1)} \neq x^{(t)}$ (\ie, the 
offspring is accepted) and that the number 
of one-bits changes in at least one block since mutations that only change the positions of one-bits 
within the blocks neither change $f$- nor $g$-value. We fix an arbitrary
search point $x^{(t)}$ and let~$A$
be the event that both $x^{(t+1)} \neq x^{(t)}$ and 
there is at least one $i\in[m]$ such that $\sigma_i\neq 0$, \ie, the number of 
one-bits changes in at least one block. Then event~$A$ requires that 
\begin{equation*} \label{eqn:sumoverwsone}
I^* \neq \emptyset \text{ and } \sum_{j \in I^*} w_j - \sum_{j \in Z^*} w_j \ge 0.
\end{equation*}

To simplify the analysis of blocks of equal weights, we from now on 
use the equivalence
\[
\sum_{j \in I^*} w_j - \sum_{j \in Z^*} w_j  = \sum_{i=1}^m \sigma_i w_{s(i)}.
\]

Hence, for $A$ to occur it is necessary that 
\begin{equation*} \label{eqn:sumoverws}
\sum_{i \mid \sigma_i > 0} \card{\sigma_i}  w_{s(i)} - \sum_{i \mid \sigma_i<0 \wedge  i\ge k} \card{\sigma_i} w_{s(i)} \ge 0,
\end{equation*}
for arbitrary $k\in[m]$ 
since we only ignore the loss due to the bits right of block $k$. In the following, $k=\kappa(i)$ will be used where $i$ is the leftmost 
block such that $\sigma_i>0$.

We now decompose
the event $A$ according to two indices~$i\in[m], \ell\in[n]$,  
where $i$ relates to the leftmost block 
 that flips more ones than zeros, and $\ell$ to the leftmost flipping one-bit 
from block~$i$.

\begin{definition}
\label{def:ai-event}
The event $A_{i,\ell}$, where $i\in [m]$ and $\ell\in[n]$, occurs iff the
following conditions hold simultaneously.
\begin{enumerate}
\item $I^*\neq \emptyset$.
\item $i \coloneqq  \max\{i \mid  \sigma_i>0 \}$.
\item $\ell = \max( I^* \cap K_i )$ 
\item $\sum_{j \mid \sigma_j>0} \card{\sigma_j} w_{s(j)} - \sum_{j\mid \sigma_j<0 \wedge j \ge \kappa(i)} \card{\sigma_j} w_{s(j)} \ge 0$.
\item A feasible search point  is obtained by flipping the bits from $I^*\cup Z^*$  in~$x^{(t)}$.
\end{enumerate}
We distinguish between two types of such events: $A_{i,\ell}$ is called \emph{potentially helpful} if $\ell>B$ and \emph{unhelpful} otherwise. 
\end{definition}

Obviously, the events $A_{i,\ell}$ are mutually disjoint. 
Since each accepted mutation that changes the value of at least one block flips at least one one-bit, 
the union of the events $A_{i,\ell}$ is a superset of~$A$ (in other words, is necessary for~$A$). If 
the leftmost flipping one-bit is among the $n-B$ most significant positions then the corresponding mutation may simultaneously flip a zero-bit 
from the $B$ least significant  positions and increase the number of one-bits at the latter positions; hence, it may be helpful. (It is 
not helpful, \eg, if the one-bit is of the same weight as position~$B$.)  
If the left-most flipping one-bit is already at the $B$ least significant positions, then the mutation cannot increase the number of one-bits at 
these positions or the resulting search point is no longer tight, and is therefore called unhelpful.

The key inequality used to bound the  drift is stated in the following lemma, which decomposes the set of 
possible events $A_{i,\ell}$ into potentially helpful and unhelpful ones.

\begin{lemma}
\label{lem:drift-under-ai}
Consider any $t\ge 0$, $i\in [m]$ and $\ell\in K_i$ such that $\Prob(A_{i,\ell})>0$. Then 
$\E{\Delta_t\mid A_{i,\ell}} \ge  0.11g_{s(i)}$ if $\ell > B$ and  $\E{\Delta_t\mid A_{i,\ell}}\ge -3/2$ otherwise.
\end{lemma}

Before we prove Lemma~\ref{lem:drift-under-ai}, let us show how it can be used to prove
Lemma~\ref{lem:drift-statement}.

\begin{proofof}{Lemma~\ref{lem:drift-statement}}
We still fix an arbitrary search point $x^{(t)}$, denote by $X_t=g(x^{(t)})$ its potential
and investigate the following step. As observed above, in the step the potential
remains either unchanged or a certain event $A_{i,\ell}$ occurs. 
Hence, the drift can be expressed as
\begin{equation*}
\E{X_t-X_{t+1} \mid X_t}
=  \sum_{i \in [m], \ell\in K_i,\Prob(A_{i,\ell})>0} \E{\Delta_{t} \mid A_{i,\ell}}\cdot  \Prob(A_{i,\ell}).
\end{equation*}
Using Lemma~\ref{lem:drift-under-ai}, the last expression is at least
\begin{align}
& \sum_{i\in [m], \ell\in K_i\cap\{B +1,\dots,n\},\Prob(A_{i,\ell})>0}
0.11g_{s(i)}  \Prob(A_{i,\ell})  \notag\\ & \qquad\qquad\qquad + \sum_{i\in [m], \ell\in K_i\cap\{1,\dots,B\},\Prob(A_{i,\ell})>0}
-(3/2)  \Prob(A_{i,\ell}).
\label{eq:drift-in-g-and-prob-ai}
\end{align}
Hence, we have to bound $\Prob(A_{i,\ell})$ from below for those events that are possible, taking into account the different 
signs of the terms. 

We first show that may basically concentrate on the case $\ell>B$ of potentially helpful mutations, more precisely, we claim 
that the bound \eqref{eq:drift-in-g-and-prob-ai} is at least 
\begin{equation}
\sum_{i\in [m], \ell\in K_i\cap\{B +1,\dots,n\},\Prob(A_{i,\ell})>0}
0.05g_{s(i)}  \Prob(A_{i,\ell}).
\label{eq:drift-in-g-and-prob-ai-1}
\end{equation}
To prove the claim, we carefully analyse $\Prob(A_{i,\ell})$ in both cases.

If $\Prob(A_{i,\ell})>0$ then there is a one-bit at position~$\ell$. If the current search point is not tight, 
already flipping bit~$\ell$ alone is accepted and we obtain $\Prob(A_{i,\ell}) \ge (1/n)(1-1/n)^{n-1}\ge 1/(en)$ as 
well as trivially $\Prob(A_{i,\ell})\le 1/n$. If the constraint is tight 
there are $B-1$ other one-bits in $x^{(t)}$. 
 Recall  that $A_{i,\ell}$ requires  $\ell$ to be  the leftmost flipping one-bit in the leftmost 
block flipping more ones than zeros. For the offspring to be feasible (Condition~5 of $A_{i,\ell}$) a zero-bit right 
of~$\ell$ must flip simultaneously with $\ell$. Let~$z$ be the number of such bits. Similarly as before, we obtain 
$\Prob(A_{i,\ell}) \ge (z/n^2)(1-1/n)^{n-1}\ge z/(en^2)$ and $\Prob(A_{i,\ell})\le z/n^2$. Altogether, for any two 
pairs $(i,\ell)$, $(i',\ell')$ where both
$\Prob(A_{i,\ell})>0$ and $\Prob(A_{i',\ell'})>0$,  
we have \begin{equation}
\frac{\Prob(A_{i',\ell'}) }{e} \le \Prob(A_{i',\ell'}) \le \Prob(A_{i,\ell}).
\label{eq:factor-e-difference}
\end{equation}
%
Note that we assume a tight, non-optimal search point. 
 Now, there is at least one possible 
events $A_{i,\ell}$ where $\ell>B$ and there are at most $B$ such events where $\ell\le B$
Using this observation and \eqref{eq:factor-e-difference}, we link each event with positive drift to at most~$B$ events with 
negative drift and 
bound \eqref{eq:drift-in-g-and-prob-ai} from below by 
\[
\sum_{i\in [m], \ell\in K_i\cap\{B +1,\dots,n\},\Prob(A_{i,\ell})>0} \left(\frac{0.11g_s(i)}{e} - B\frac{3}{2}\right).
\]
Since $g_\ell\ge 75B$ for $\ell>B$, we have $\frac{0.11g_s(i)}{e} - B (3/2) \ge \frac{0.11g_s(i)}{e} - g_s(i)/50 \ge 0.05g_i$, so altogether  
the bound \eqref{eq:drift-in-g-and-prob-ai} on the drift  is at least 
\[
\sum_{i\in [m], \ell\in K_i\cap\{B +1,\dots,n\},\Prob(A_{i,\ell})>0} 0.05g_{s(i)} \Prob(A_{i,\ell})
\]
as claimed.

We proceed by bounding \eqref{eq:drift-in-g-and-prob-ai-1} further from below. To this end, we derive two different lower 
bounds on $\Prob(A_{i,\ell})>0$, where $\ell>B$. 
Considering a mutation that flips 
bit~$\ell>B$, the mutation is accepted if it flips a zero-bit right of~$\ell$ and does not flip any further bits. 
Firstly,  even if 
all $B$ one-bits are right of (and including) bit~$\ell$, there are at least $\ell-B$ zero-bits right of~$\ell$. 
Secondly, denoting by  
$\Hell\coloneqq \sum_{i=1}^{B} x_i^{(t)}$  the number of 
zero-bits 
among the $B$ rightmost positions,  
there are at least $\Hell$ zero-bits right of~$\ell$. 
%
Counting the at least $\min\{\ell-B,\Hell\}$ different ways of flipping a zero-bit right of~$\ell$,  
we conclude that 
\begin{equation}
\Prob(A_{i,\ell})\ge \frac{\max\{\ell-B, \Hell\}}{n^2}\left(1-\frac{1}{n}\right)^{n-2}\ge \frac{\max\{\ell-B,\Hell\}}{en^2}
\label{eq:bound-ai-ell}
\end{equation} if $A_{i,\ell}$ is possible.

 We will now relate 
the expression $\ell-B$ to the factor $g_{s(i)}$ appearing in~\eqref{eq:drift-in-g-and-prob-ai-1}. First of all, since 
$\ell$ appears in block~$i$ and all bits in a block have equal weight, we have $g_{s(i)}=g_{\ell}$. 
Next 
 we note that 
$g_{\ell} \le \gamma_{\ell} = 75B(\ell-B)^7$ by Definition~\ref{def:potential} for $\ell>B$, 
so $\ell-B \ge (g_{s(i)}/75)^{1/7}\ge (1/2)(g_{s(i)})^{1/7}B^{-1/7}$. 
Plugging this into Equation~\eqref{eq:bound-ai-ell}, 
we obtain (if $A_{i,\ell}$ is possible) that 
\begin{equation}
\Prob(A_{i,\ell})\ge \frac{\max\{(g_{s(i)})^{1/7}/(2B^{1/7}),  \Hell\} }{en^2}.
\label{eq:prob-ai-ell-bound-oneseventh-max}
\end{equation}


We will now apply \eqref{eq:prob-ai-ell-bound-oneseventh-max} to obtain our final bound on \eqref{eq:drift-in-g-and-prob-ai-1}.  
%
Let $\tilde{I}=\{x_i\mid x_i=1\wedge i>B\}$ be the set of ones-bits at the $n-B$ leftmost positions.  
Since for each~$i\in [m]$ 
there are $\card{K_i\cap \tilde{I}}$ disjoint events $A_{i,\ell}$ each, namely one
 for each one-bit~$\ell$ within block~$i$,  
we obtain by combining \eqref{eq:drift-in-g-and-prob-ai-1} and \eqref{eq:prob-ai-ell-bound-oneseventh-max} that
\begin{align}
& \E{X_t-X_{t+1} \mid X_t}  \notag\\
& \ge 
\sum_{i\in [m]\mid K_{i}\cap \tilde{I} \neq \emptyset} \card{K_i\cap \tilde{I}} 0.05g_{s(i)} \Prob(A_{i,\ell})
\notag\\
& \ge 
\sum_{i\in [m]\mid K_{i}\cap \tilde{I} \neq \emptyset} \card{K_i\cap \tilde{I}}\frac{ 0.025 \max\{(g_{s(i)})^{8/7}B^{-1/7}, g_{s(i)}\cdot \Hell\} }{en^2}.
\label{eq:e-ai-ell-bound-max}
\end{align}
Using the estimate \[
\sum_{i\in \tilde{I}}(g_i)^{8/7} \ge \left(\sum_{i\in \tilde{I}}(g_i)\right)^{8/7} (\card{\tilde{I}})^{-1/7} \ge B^{-1/7} \left(\sum_{i\in \tilde{I}}(g_i)\right)^{8/7}\]
proved in Lemma~\ref{lem:convexity} 
and recalling \eqref{eq:g-respects-ones-and-zeros-1}
 we finally have 
\begin{align*}
 \E{X_t-X_{t+1} \mid X_t}  
& \ge \frac{0.025}{en^2} \min\left\{\frac{(g(x^{(t)}))^{8/7}  }{B^{2/7}}, g(x^{(t)})\right\} \\ 
		& = 
\frac{0.025g(x^{(t)}}{en^2} \min\left\{\frac{(g(x^{(t)}))^{1/7}  }{B^{2/7}}, 1\right\}.
\end{align*}
This proves the first statement of Lemma~\ref{lem:drift-statement}.

For the second statement of Lemma~\ref{lem:drift-statement}, we simply 
use that $g_i\le 75Bi^7$, so for all $x^{(t)}$ it holds that 
$g(x^{(t)})\le \sum_{i=1}^n g_i \le nB\cdot 75n^7 \le 75n^9$. Also, since $g_i \ge 1$ 
for $i\in [n]$, 
each non-optimal search point $x^{(t)}$ must satisfy $g(x^{(t)})\ge 1$.\qed\end{proofof}

The still outstanding proof of Lemma~\ref{lem:drift-under-ai}
requires a careful analysis of the one-step drift, taking
into account the specific structure of the drift function.

\begin{proofof}{Lemma~\ref{lem:drift-under-ai}}
Recall that we want to condition on the event~$A_{i,\ell}$ (Definition~\ref{def:ai-event}),
where $i$ is the leftmost block flipping more ones than zeros.
Moreover, recall the notions introduced 
in Definitions~\ref{def:potential} and \ref{def:idiuetc}.
Let
\begin{equation*}
\begin{split}
\Delta_L(i) & \coloneqq \left(\sum_{j\mid \sigma_j > 0}  \card{\sigma_j}  g_{s(j)} - 
  \sum_{j\mid \sigma_j<0 \wedge j \ge \kappa(i)} \card{\sigma_j}  g_{s(j)}\right) \cdot\indic{A} , \\
\Delta_R(i) & \coloneqq \left(\sum_{j\mid \sigma_j>0 \wedge j< \kappa(i)} \card{\sigma_j} g_{s(j)}\right) \cdot\indic{A},
\end{split}
\end{equation*}
where $\indic{A}$ denotes the indicator random variable of event~$A$. 
Recall that $\Delta_t=0$ if $A$ does not occur. Otherwise, 
$\Delta_t=\sum_{j\mid \sigma_j>0} \card{\sigma(j)} g_{s(j)} - \sum_{j\mid \sigma_j<0} \card{\sigma(j)} g_{s(j)}$. Hence,  
we have $\Delta_t = (\Delta_L(i) - \Delta_R(i))$ for all~$i\in[m]$. By 
linearity of expectation, we obtain
\begin{equation} \label{eqn:deltaexpectation}
 \E{ \Delta_t \mid A_{i,\ell}}
=  \E{\Delta_L(i) \mid A_{i,\ell} } - \E{\Delta_R(i)  \mid A_{i,\ell} }.
\end{equation}

We first show that $\left( \Delta_L(i) \mid A_{i,\ell} \right)$ is
a non-negative random variable, \ie, the probability of any negative outcome is~$0$.
To prove this, assume that $A_{i,\ell}$ holds, which implies that no block left of~$i$ 
flips more ones than zeros.

We now inspect the relation between the weights of the original function and the potential function. Here 
we exploit that the ratio of $g$-values and $w$-values of two blocks $i>j$ is the same unless the weight 
of block $i$ is capped by the minimum operator in the definition of $g_{s(i)}$ in Definition~\ref{def:potential}. Otherwise, 
the ratio may be smaller. 
Looking also into symmetrical cases, for any $i\in\ [m]$ we obtain from Definition~\ref{def:potential} that 
\begin{gather}
\label{eq:weightcapsthree}
\frac{g_{s(j)}}{g_{s(\kappa(i))}} = \frac{w_{s(j)}}{w_{s(\kappa(i))}} \text{ for $i\ge j \ge\kappa(i)$,}\\
\label{eq:weightcaps}
\frac{g_{s(j)}}{g_{s(\kappa(i))}} \le \frac{w_{s(j)}}{w_{s(\kappa(i))}} \text{ for $j\ge \kappa(i)$,}\\
\label{eq:weightcapstwo}
\frac{g_{s(j)}}{g_{s(\kappa(i))}} \ge \frac{w_{s(j)}}{w_{s(\kappa(i))}} \text{ for $j <\kappa(i)$.}
\end{gather}

Hence,
\begin{equation*}
\begin{split}
& (\Delta_L(i) \mid A_{i,\ell})  = 
\left(\sum_{j\mid \sigma_j > 0}  \card{\sigma_j}  g_{s(j)} - 
  \sum_{j\mid \sigma_j<0 \wedge j \ge \kappa(i)} \card{\sigma_j}  g_{s(j)}\right) \\
& \ge 
\left(\sum_{j\mid \sigma_j > 0}  \card{\sigma_j}  g_{s(\kappa(i))} \frac{w_{s(j)}}{w_{s(\kappa(i))}}
- \sum_{j\mid \sigma_j<0 \wedge j \ge \kappa(i)} \card{\sigma_j} g_{s(\kappa(i))} \frac{w_{s(j)}}{w_{s(\kappa(i))}} \right)\\
& \ge 0,
\end{split}
\end{equation*}
where the first inequality uses \eqref{eq:weightcapsthree}--\eqref{eq:weightcapstwo} along with the fact that 
no block left of $i$ has positive $\sigma$-value, and 
the last inequality holds by the fourth item from the definition of~$A_{i,\ell}$ (Definition~\ref{def:ai-event}).

We note that according to the fifth item of Definition~\ref{def:ai-event}, 
this event 
may imply 
that a bit~$j^*\in Z$ 
flips to~$1$ simultaneously with a one-bit in block~$i$ flipping to~$0$. 
This is the case if the constraint is tight in the search point $x^{(t)}$, which we again 
pessimistically 
assume to be the case (if $x^{(t)}$ had more than $B$ one-bits, flipping only $\ell$ would already be accepted).

In the following, we concentrate on the case $\ell>B$, \ie, the case of a potentially helpful 
mutation, and consider the other case at the end of this proof. 
Now let $S_{i,\ell}$ be the event that the following three events happen simultaneously:
\begin{enumerate}
\item 
 $\card{\{I^*\cup Z^*\}\cap K_j} = 0 $ for all $j \in\{\kappa(i),\dots,m\}\setminus\{i\}$
\item  $\card{ I^*\cap K_i }=1$ and $\ell\in I^*\cap K_i$, 
\item  $\card{ Z^*\cap K_i }=0$, 
\end{enumerate}
\ie, block $i$ is the only one in $L(i)$ that 
contributes to $\Delta_L$ by flipping exactly one one-bit at position~$\ell$.
 We have
\begin{multline*}
\E{\Delta_L(i) \mid A_{i,\ell} }  = \E{ \Delta_L(i) \mid A_{i,\ell} \cap S_{i,\ell} } \cdot \Pr(S_{i,\ell} \mid A_{i,\ell}) \\
  + \E{\Delta_L(i) \mid A_{i,\ell} \cap \bar{S}_{i,\ell} } \cdot \Pr(\bar{S}_{i,\ell} \mid A_{i,\ell})
\end{multline*}
by the law of total probability. As the random variable $\left( \Delta_L(i) \mid A_{i,\ell}
\right)$ cannot have any negative outcomes, all these conditional expectations are
non-negative as well. From~\eqref{eqn:deltaexpectation} we thus derive
\begin{equation} 
\E{\Delta_t \mid A_{i,\ell}} \ge \E{\Delta_L(i) \mid A_{i,\ell} \cap S_{i,\ell} }\cdot \Pr(S_{i,\ell} \mid A_{i,\ell})
 - \E{\Delta_R(i) \mid A_{i,\ell}}.
\label{eqn:expectationcentral}
\end{equation}

We will now bound the terms from \eqref{eqn:expectationcentral} from below to obtain our
result.
  For $(S_{i,\ell} \mid A_{i,\ell})$ to occur, it is
sufficient  that all bits in the blocks in $L(i)$ except the one-bit~$\ell$ in block~$i$ and bit~$j^*$ 
 do not flip (note that these bits flip since we condition on~$A_{i,\ell}$). 
 Consequently, $\Pr(S_{i,\ell} \mid A_{i,\ell})
 \ge (1-1/n)^{n-2} \ge 1/e$. 
Moreover, since no zero-bits in $L(i)$ flip under $A_{i,\ell}\cap S_{i,\ell}$, $j^*$ must be in a block in $R(i)$. 
Hence, $\E{\Delta_L(i) \mid A_{i,\ell}
  \cap S_{i,\ell} } \ge g_{s(i)}$. Altogether, 
	\begin{equation}
	\label{eq:leftcontrib}
	\E{ \Delta_L(i) \mid A_{i,\ell} \cap S_{i,\ell} } \cdot \Pr(S_{i,\ell} \mid A_{i,\ell}) \ge \frac{g_{s(i)} }{e} .
	\end{equation}

		Finally, we need a bound on~$\E{\Delta_R(i)\mid A_{i,\ell}}$, which is determined by the bits in~$R(i)$ that flip to~$1$, \ie, 
		bits 
	from blocks $1,\dots,\kappa(i)-1$. Note that event $A_{i,\ell}$ 
	might imply that at least one of these bits flips to~$1$ for sure to maintain feasibility of the search 
	point. We still pessimistically assume this to happen and denote by $j^*$ the random index of the 
	zero-bit that is forced to flip.
  Furthermore, since $\ell>B$ and exploiting the monotonicity of the weights $g_k$, 
	we pessimistically assume that bits $1,\dots,B-1$ are all~$1$ in~$x^{(t)}$  
so that the contribution of bit~$j^*$ becomes as large as possible. Then, since the flips 
in $R(i)$ are not part of the fourth item in the definition of~$A_{i,\ell}$ (Definition~\ref{def:ai-event}), 
we conclude that $j^*$ is 
uniform on $\{B,\dots,s(\kappa(i))-1\}$
and contributes at most 
\[
\frac{1}{s(\kappa(i))-B}\sum_{k=B}^{s(\kappa(i))-1} g_k. 
\]
With respect to the bits different from $j^*$, we exploit that they 
 are flipped independently. Hence, on~$A_{i,\ell}$, the probability 
that 
$k\in Z\cap\{B,\dots,s(\kappa(i))-1\}\setminus\{j^*\}$ flips is bounded from above by $\frac{1}{n}$. 
Pessimistically, we assume that $A$ occurs in such a mutation. By  using linearity of expectation and 
combining with the contribution of~$j^*$, it follows that 
\[
\E{\Delta_R(i) \mid A_{i,\ell} } \le  \sum_{k=B}^{s(\kappa(i))-1} \frac{1}{n }g_k + 
\frac{1}{s(\kappa(i))-B}\sum_{k=B}^{s(\kappa(i))-1} g_k,
\]
which is at most
\[
\frac{2}{s(\kappa(i))-B}\sum_{k=B}^{s(\kappa(i))-1} g_k \le \frac{2}{s(\kappa(i))-B}\sum_{k=B}^{s(\kappa(i))-1} \gamma_k,
\]
where we used that $g_k\le \gamma_k$ for all $k\in[n]$ by Definition~\ref{def:potential}. 
Along with~\eqref{eqn:expectationcentral} and~\eqref{eq:leftcontrib}, we obtain
\begin{align}
 \E{ \Delta_t \mid A_{i,\ell}}  
 & \ge \frac{g_{s(i)}}{e} - \frac{2}{s(\kappa(i))-B}\sum_{k=B}^{s(\kappa(i))-1} \gamma_k.
\notag
\label{eq:centraltwo}
\end{align}

We are left with the sum over~$k$. Plugging in the definition of~$\gamma_k$ and taking care of its different cases, this is estimated by
\begin{align*}
 \sum_{k=B}^{s(\kappa(i))-1} \gamma_k & 
 \le 1 + \sum_{k=B+1}^{s(\kappa(i))-1} 75B(k-B)^7 \\
& \le 1 + \frac{75B}{8}\left((s(\kappa(i))-B)^8-1^8\right) \le \frac{75B(s(\kappa(i))-B)^8}{8} \\
& \le 
\frac{(s(\kappa(i))-B) g_{s(i)}}{8},  
\end{align*}
where we used that $g_{s(\kappa(i))} = \gamma_{s(\kappa(i))} = 75B(s(\kappa(i))-B)^7$ according to the definition 
of $\kappa(i)$ as well as $g_{s(i)} \ge g_{s(\kappa(i))}$.

Hence, altogether, 
\begin{align*}
\E{ \Delta_t \mid A_{i,\ell}}  & \ge
 \frac{g_{s(i)} }{e} - 
\frac{2g_{s(i)}(s(\kappa(i))-B)}{8(s(\kappa(i))-B)} \ge 0.11 g_{s(i)},
\end{align*}
which concludes the proof in the case $\ell>B$.

We are left with the case $\ell\le B$, \ie, an unhelpful mutation.
Recalling that we work under $A_{i,\ell}$, we 
note that $\ell$ cannot be from block~$1$ since 
then $1$~would be the number of the leftmost block that flips more ones than zeros, in contradiction with 
the fourth case of the definition of~$A_{i,j}$. Hence, there is a zero-bit~$j^*$ right of $\ell$ in a lower-numbered block 
that flips to~$1$ simultaneously with~$\ell$ flipping to~$0$. If no other one-bit flips then zero-bits left of~$\ell$ cannot flip simultaneously 
with the pair $(\ell,j^*)$; however, if there are further flips of one-bits among the $B$ rightmost positions (recalling that $\ell$ 
is the left-most flipping one-bit) then 
there might be zero-bits flipping left of~$\ell$. Let $S$ be the number of additionally flipping one-bits among the $B$ 
rightmost positions, \ie, 
there are $S+1$ flipping one-bits there.  
For the mutation to be accepted (and $A_{i,\ell}$ to occur), the weight $w_z$ of each potentially 
flipping zero-bit~$z$ (where $z>\ell$ is possible) cannot exceed $(S+1)w_\ell$. Due to~\eqref{eq:weightcaps}, 
we have $g_z/g_\ell \le w_z/w_\ell$. Along with $g_\ell=1$ for $\ell\le B$, we have 
$g_z\le S+1$. Hence, the expected contribution 
of these bits to $\Delta_t$ is no less than $-(S+1)$. 
If we can show that $\Prob(S=s)\le 2^{-s}$, then we altogether have in the case $\ell\le B$ that 
\[
\E{ \Delta_t \mid A_{i,\ell}} \ge -\sum_{s=1}^\infty (s+1)2^{-s} \ge -3/2
\]
as claimed. To conclude the proof, we note that (still on $A_{i,\ell}$) $S=s$ can only happen if $s$ one-bits and $s$ zero-bits flip simultaneously 
and in addition to the flipping one-bit at position~$\ell$.
The probability of this happening is maximized if there are $n/2$ zero- and one-bits and is therefore at most 
\[
\binom{n/2}{s}\binom{n/2}{s} \frac{1}{n^{2s}} \le \left(\frac{(n/2)^s}{s!}\right)^{2} \frac{1}{n^{2s}} 
\le \frac{1}{2^{s^2} (s!)^2} \le 2^{-s}
\]
as suggested.\qed 
\end{proofof}

\section{A Tight Bound for Specific Instances}
As mentioned above, we can show a tight runtime bound of $O(n^2)$ for instances where the $B$ least significant weights 
are identical, \ie, $w_B=w_1$; this includes the case that the function equals $\onemax$. The analysis is in very large 
tracks identical to the one from the previous sections proving 
Theorem~\ref{theo:upper}  
such that we only describe the places where changes are necessary. We prove the following theorem.

\begin{theorem}
\label{theo:upper-modified}
Let a linear function $f(x)=w_1x_1+\dots w_nx_n$, where $w_n\ge \dots \ge w_1$,  
under a uniform constraint $x_1+\dots+x_n\ge B$ for $B\in\{1,\dots,n\}$ be given. If $w_1=w_B$ then
the expected optimisation time of the  \ea optimizing $f$ under the constraint 
is upper bounded by $O(n^2)$. Also, the time is  $O(n^2 \log n)$ with probability $1-O(n^{-c})$ 
for any constant~$c>0$.
\end{theorem}

From now on, we assume that we have been given a fitness function with the property $w_1=w_B$. 
The key idea for the proof of Theorem~\ref{theo:upper-modified} is as follows:  since $w_1=w_B$,  
block~$1$ as defined in Definition~\ref{def:potential} includes the positions $1,\dots,B$ and, if $s(2)>B+1$, 
some positions further to the left, 
more precisely $B+1,\dots,s(2)-1$. Clearly, all events $A_{1,\ell}$, \ie, where the leftmost block flipping more ones than zeros is block~$1$, 
are impossible since they would increase the $f$\nobreakdash-value. 
Hence, the previously considered so-called unhelpful mutations (\ie, mutations where the leftmost flipping one-bit in blocks 
flipping more ones than zeros is among the $B$ rightmost positions) are impossible and we do no longer 
have to handle the negative drift that could arise from such mutations in the context of 
Theorem~\ref{theo:upper-modified}. In turn, this now allows us to use a potential function that assigns smaller weights 
to the positions $B+1,\dots,n$ than in the general case.


The modified potential function used here is constructed in the same way as before but uses the following new $\gamma$-values:
For $j\in [n]$, we let 
\begin{equation*} \label{eqn:gamma_i-modified-1}
\gamma_j \coloneqq 
\begin{cases}
1 & \text{ if $j\le B$,}\\
 8(j-B)^7  & \text {otherwise}.
\end{cases}
\end{equation*}
Hence, the $\gamma$-values in the second case are by a factor~$\Theta(B)$ smaller compared to Definition~\ref{def:potential}. 


Using the modified potential function, which we still call~$g$, we obtain the following analogue of Lemma~\ref{lem:drift-under-ai}.

\begin{lemma}
\label{lem:drift-under-ai-mod}
Consider any $t\ge 0$, $i\in [m]$ and $\ell\in K_i$ such that $\Prob(A_{i,\ell})>0$. Then 
$\E{\Delta_t\mid A_{i,\ell}} \ge  0.011g_{s(i)}$ if $\ell > B$ and  $\E{\Delta_t\mid A_{i,\ell}}\ge 0$ otherwise.
\end{lemma}

\begin{proof}
We describe the required changes compared to the proof of Lemma~\ref{lem:drift-under-ai}. As argued before, 
unhelpful mutations are not possible in the case $w_1=w_B$ so that $\E{\Delta_t\mid A_{i,\ell}}\ge 0$ for $\ell \le B$. 

We are left  with the case $\ell>B$. The analysis differs only at the point where we estimate the 
contribution to the drift of the bits in~$R(i)$. Again, 
this contribution is 	bounded from above by 
\[
\frac{2}{s(\kappa(i))-B}\sum_{k=B}^{s(\kappa(i))-1} \gamma_k,
\]
which, using our new definition of~$\gamma_k$, is at most
\begin{align*}
\sum_{k=B}^{s(\kappa(i))-1} \gamma_k & 
 \le 1 + \sum_{k=B+1}^{8s(\kappa(i))-1} 8(k-B)^7 \\
& \le 1 + \frac{1}{8}\left(8(s(\kappa(i))-B)^8-1^8)\right) \le (s(\kappa(i))-B)^8\\
& \le 
\frac{(s(\kappa(i))-B) g_{s(i)}}{8},  
\end{align*}
where we used that $g_{s(\kappa(i))} = \gamma_{s(\kappa(i))} = 8(s(\kappa(i))-B)^7$ according to the definition 
of $\kappa(i)$ as well as $g_{s(i)} \ge g_{s(\kappa(i))}$.

Along with~\eqref{eqn:expectationcentral} and~\eqref{eq:leftcontrib}, which still hold for the modified 
potential function, we have 
\begin{align*}
\E{ \Delta_t \mid A_{i,\ell}}  & \ge
 \frac{g_{s(i)} }{e} - 
\frac{2g_{s(i)}(s(\kappa(i))-B)}{8(s(\kappa(i))-B)} \ge 0.11 g_{s(i)}
\end{align*}
as suggested.
\end{proof}

Using Lemma~\ref{lem:drift-under-ai-mod}, we next show the following lemma, which corresponds to 
Lemma~\ref{lem:drift-statement}.

\begin{lemma}
\label{lem:drift-statement-mod}
Considering a random variable $X_t= g(x^{(t)})$, with the modified potential function $g$ defined above, 
and $x^{(t)}$ is the random search point of the \ea at time $t$, for all time steps $t$ we have
\begin{enumerate}
\item
$\E{X_t-X_{t+1} \mid X_t} \ge 
\frac{ 0.055 X_t^{15/14}} { en^2}$.
\item
$1 \le X_t = O(n^8) $ if $x^{(t)}$ is not optimal.
\end{enumerate}
\end{lemma}

\begin{proof}
The proof starts in the same ways as the one of Lemma~\ref{lem:drift-statement-mod} (hereinafter called the \emph{original proof}). The drift 
 can be expressed as
\begin{equation*}
\E{X_t-X_{t+1} \mid X_t}
=  \sum_{i \in [m], \ell\in K_i,\Prob(A_{i,\ell})>0} \E{\Delta_{t} \mid A_{i,\ell}}\cdot  \Prob(A_{i,\ell}).
\end{equation*}
Using Lemma~\ref{lem:drift-under-ai-mod}, the last expression is at least
\begin{align}
& \sum_{i\in [m], \ell\in K_i\cap\{B +1,\dots,n\},\Prob(A_{i,\ell})>0}
0.11g_{s(i)}  \Prob(A_{i,\ell})  ,
\label{eq:drift-in-g-and-prob-ai-mod}
\end{align}
which is the first difference to the original proof since we no longer have to consider unhelpful mutations.

Proceeding as in the original proof, we still have 
\begin{equation}
\Prob(A_{i,\ell})\ge \frac{\max\{\ell-B, \Hell\}}{n^2}\left(1-\frac{1}{n}\right)^{n-2}\ge \frac{\max\{\ell-B,\Hell\}}{en^2}
\label{eq:bound-ai-ell-mod}
\end{equation} if $A_{i,\ell}$ is possible. As in the original proof, we  now relate 
the expression $\ell-B$ to the factor $g_{s(i)}$ appearing in~\eqref{eq:drift-in-g-and-prob-ai-mod}. Adjusting the 
to the new $\gamma$-value, we have   
$\ell-B \ge (g_{s(i)}/8)^{1/7}\ge (1/2)(g_{s(i)})^{1/7}$. 
Plugging this into Equation~\eqref{eq:bound-ai-ell-mod}, 
we obtain (if $A_{i,\ell}$ is possible) that 
\begin{equation}
\Prob(A_{i,\ell})\ge \frac{\max\{(g_{s(i)})^{1/7}/2,  \Hell\} }{en^2}.
\label{eq:prob-ai-ell-bound-oneseventh-max-mod}
\end{equation}
With the same arguments as in the original proof, we therefore have 
\begin{align}
& \E{X_t-X_{t+1} \mid X_t}  \notag\\
& \ge 
\sum_{i\in [m]\mid K_{i}\cap \tilde{I} \neq \emptyset} \card{K_i\cap \tilde{I}} 0.11g_{s(i)} \Prob(A_{i,\ell})
\notag\\
& \ge 
\sum_{i\in [m]\mid K_{i}\cap \tilde{I} \neq \emptyset} \card{K_i\cap \tilde{I}}\frac{ 0.055 \max\{(g_{s(i)})^{8/7}, g_{s(i)}\cdot \Hell\} }{en^2},
\label{eq:e-ai-ell-bound-max-mod}
\end{align}
which differs from \eqref{eq:e-ai-ell-bound-max} solely because of the modified potential function.
Deviating from the original proof, we now distinguish between two cases according to $\Hell$.

\textbf{Case~$1$:} $\Hell \le \sqrt{g(x^{(t)})}$. Then we use the term $(g_{s(i)})^{8/7}$ from the maximum in 
\eqref{eq:e-ai-ell-bound-max-mod} and obtain 
\begin{align*}
& \E{X_t-X_{t+1} \mid X_t}  \\
& \ge 
\sum_{i\in [m]\mid K_{i}\cap \tilde{I} \neq \emptyset} \card{K_i\cap \tilde{I}}\frac{ 0.055 (g_{s(i)})^{8/7} }{en^2} \\
& \ge \frac{0.055 (g(x^{(t)})^{8/7}  )}{\Hell^{1/7}en^2} 
 \ge\frac{0.055 g(x^{(t)})  }{2en^2} ,
\end{align*}
where we used the estimate \[
\sum_{i\in \tilde{I}}(g_i)^{8/7} \ge \left(\sum_{i\in \tilde{I}}(g_i)\right)^{8/7} (\card{\tilde{I}})^{-1/7} \ge B^{-1/7} \left(\sum_{i\in \tilde{I}}(g_i)\right)^{8/7}\]
proved in Lemma~\ref{lem:convexity} 
along with \eqref{eq:g-respects-ones-and-zeros-1}.
%
Using the assumption $\Hell\le \sqrt{g(x^{(t)})}$ from the case analysis, we finally have 
\begin{align*}
& \E{X_t-X_{t+1} \mid X_t}  
\ge \frac{0.055 (g(x^{(t)}))^{8/7}  }{g(x^{(t)})^{1/14}en^2} = 
\frac{0.055 (g(x^{(t)}))^{15/14}  }{en^2}.
\end{align*}

\textbf{Case~$2$:} $\Hell> \sqrt{g(x^{(t)})}$. Then we use the term $\Hell$ from the maximum 
in \eqref{eq:e-ai-ell-bound-max-mod} and obtain 
\begin{align*}
& \E{X_t-X_{t+1} \mid X_t}  \\
& \ge 
\frac{0.055 g(x^{(t)}) h(x^{(t)}  )}{en^2} \ge 
\frac{0.055 g(x^{(t)}) \sqrt{g(x^{(t)})}}{en^2}\\
& = 
\frac{0.055 (g(x^{(t)}))^{3/2}  }{en^2}.
\end{align*}

Altogether, 
we therefore have
\[
\E{X_t-X_{t+1} \mid X_t}   \ge \frac{0.055(X_t)^{15/14}}{en^2}. 
\]
%
%
%
This proves the first statement of Lemma~\ref{lem:drift-statement-mod}. The second statement is proved in the same way as in 
Lemma~\ref{lem:drift-statement-mod}, taking into account the smaller weights of the modified potential function \qed\end{proof}

With these two lemmas in place, we show our result.

\begin{proofof}{Theorem~\ref{theo:upper-modified}}
We apply the variable drift theorem (Theorem~\ref{theo:variableDrift}) given 
the statements of Lemma~\ref{lem:drift-statement-mod}. 
Using that $X_t\ge 1\eqqcolon \smin$ 
and $X_t=O(n^8)$ as well as the drift bound 
 \[h(X_t)\coloneqq \frac{ 0.055 (X_t)^{15/14}} { en^2} \text{,}\]
the expected optimisation time  is bounded by 
\begin{align*}
& \frac{\smin}{h(\smin)} + \int_{\smin}^{n^8} \frac{1}{h(x)} \,\mathrm{d}x \\ 
& = O(n^2) + \frac{en^2}{0.055}\left( 
\int_{1}^{n^8} \frac{1}{x^{15/14}} \mathrm{d}x \right)\\
& = O(n^2) + O(n^2) = O(n^2),
\end{align*}
which completes the proof of the bound on the expected time.

For the tail bound we use the multiplicative drift theorem (Theorem~\ref{theo:multdrift-upper}) 
with the simple bound $\E{X_t-X_{t+1} \mid X_t} \ge \frac{ 0.055 X_t} { en^2} $
 from  Lemma~\ref{lem:drift-statement}, along with $X_t = O(n^8)$ that implies $\ln(X_t/\smin)=O(\log n)$. 
Note that the multiplicative drift theorem gives the  upper bound $O(n^2\log n)$ on the expected optimisation time so that  the tail bound  
can be obtained by setting $r=c\ln n$.\qed
\end{proofof}

We note that upper bound given in Theorem~\ref{theo:upper-modified} is tight. 
The lower bound of $\Omega(n^2)$ 
for the \ea is similar to Theorem~\ref{theo:RLSLowerBound} and follows by the fact 
that a special $2$-bit flip is needed if the current solution is non-optimal and has 
exactly $B$ 1-bits. For a detailed analysis see Theorem~10 in~\cite{FRIEDRICH2018}.

\section{Minimisation versus Maximisation}
\label{sec:min-vs-max}
The results we presented in this paper have been formulated with respect 
to the minimisation of linear functions under a uniform lower 
constraint $x_1+\dots+x_n\ge B$. This perspective of minimisation fits more naturally 
the minimisation of potential functions used in drift theorems (Theorems~\ref{theo:variableDrift} 
and~\ref{theo:multdrift-upper}) and is therefore de-facto standard in many recent papers dealing 
with the optimisation of linear functions \cite{DJWLinearRevisited, WittCPC13}.

However, previous works about the optimisation of linear functions under constraints considered 
the maximisation of a linear function under an upper uniform constraint $x_1+\dots+x_n\le B$. It is not difficult 
to see that our main theorems (Theorems~\ref{theo:RLS} and~\ref{theo:upper}) also hold for this scenario. 
Since this is rather straightforward to realise for \rls, we only discuss the result for the \ea now. 
The potential function~$g$ from Definition~\ref{def:potential} would have to be adapted to assign weight~$0$ to the 
$B$ most significant positions and increasing weights from bits $1$ to $n-B$ in the same way as before, with the exception 
that $0$-bits instead of $1$-bits would contribute: roughly speaking we would define 
$g(x)\coloneqq \sum_{i=1}^{n-B} g_i(1-x_i)$. 

Interestingly, the time to reach the feasible region analysed in Lemma~\ref{lem:RLSFeasible} will be 
$O(n\log(n/B))$ instead of $O(n\log(n/(n-B))$ 
in the scenario of maximisation, as proved in earlier work \cite{FRIEDRICH2018}. This is due to the fact 
that a large~$B$ corresponds to a large infeasible region in the maximisation case but a small one in the 
minimisation case. However, in both cases the time to reach the feasible region is always bounded by an asymptotically 
smaller expression than our bound for the time to find an optimal search point 
after having reached the feasible region.

\section{Experimental Supplements}
\label{sec:experiments}

We carry out experimental investigations for the (1+1) EA that complement the theoretical results  provided in this paper.
%
Our experiments provide additional insights with respect to two aspects. Firstly, we investigate the runtime of the (1+1)~EA in dependence of the given constraint bound $B$. Secondly, we study higher mutation rates than $1/n$, namely  $2/n$ and $3/n$.  We consider the special case of the objective function where $w_i=i$, $1 \leq i \leq n$. The objective function has the important property that all bits have different weights. For a given bound $B$ the unique optimal solution consists of the first $B$ bits.

\begin{figure}

\includegraphics[trim=1cm 2.0cm 0cm 2cm, clip=true, scale=0.405]{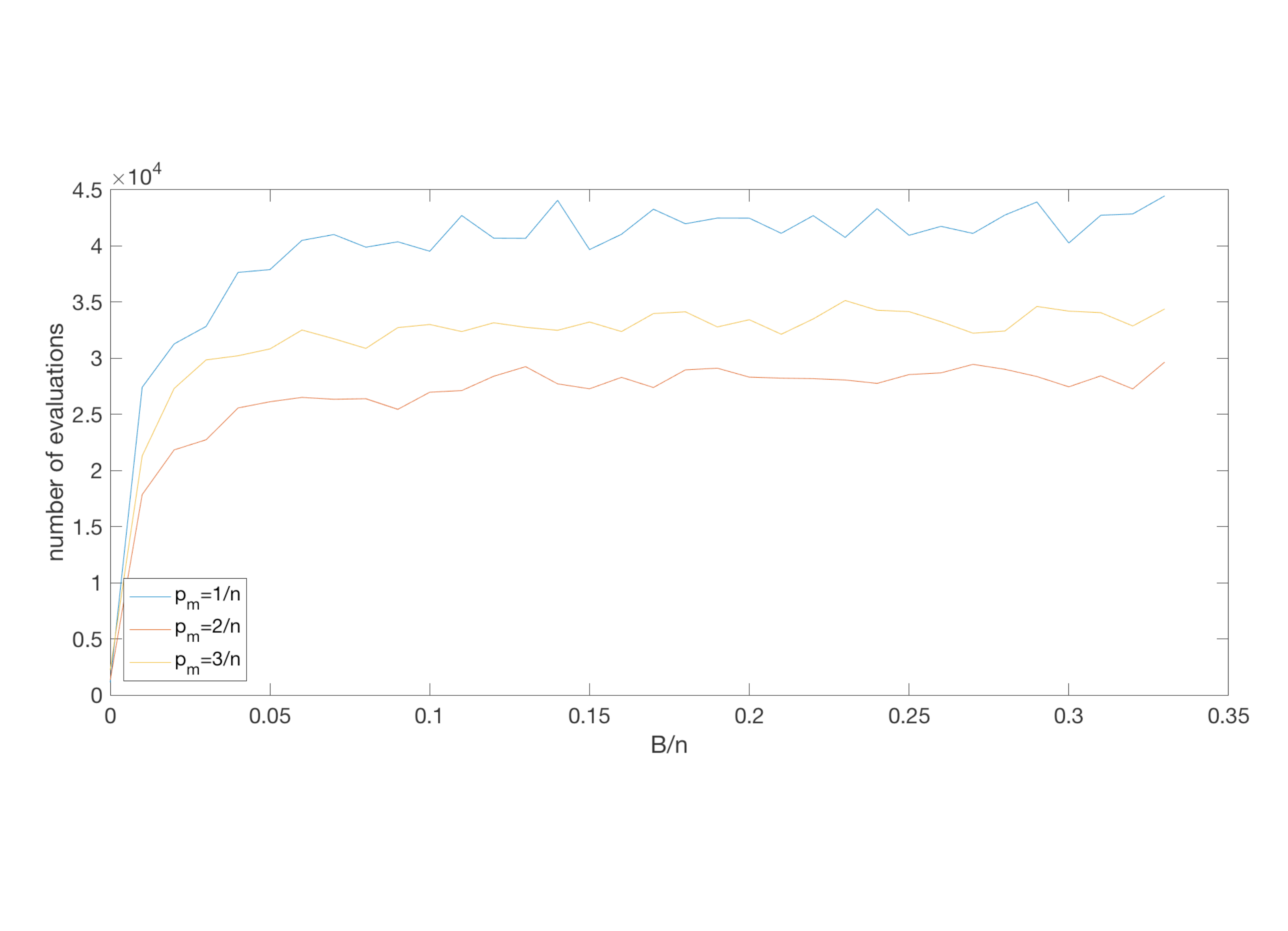}  
\includegraphics[trim=0cm 2.0cm 0cm 2cm, clip=true, scale=0.4]{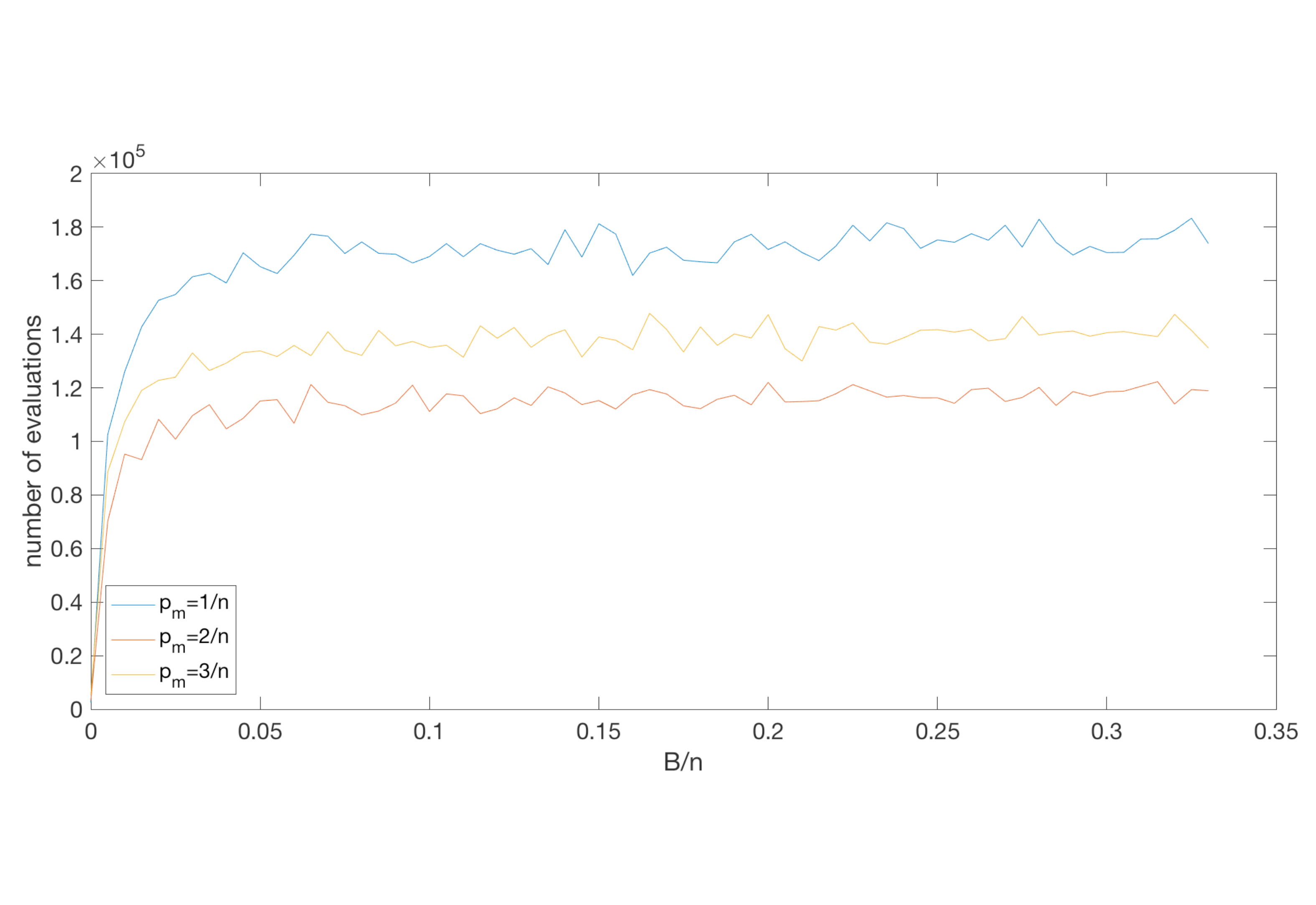}

\caption{Experimental results for $n=100, 200$ (top and bottom) and $B=0, 1, \dots n/3$ using (1+1) EA with mutation rate $p_m= 1/n, 2/n, 3/n$.}
\label{fig:exp1}
\end{figure}

Figure~\ref{fig:exp1} shows our results for $n=100$ and $n=200$. For each value of $B=0, \ldots, n/3$, we carried out $500$ independent runs and the figure shows for each $B$ the average value of these $500$ runs. For $B=0$, the expected runtime of $(1+1)$ EA is $\Theta(n \log n)$ whereas for $B>0$ this value becomes $O(n^2\log^+ B)$ due to the upper bound given in Theorem~\ref{theo:upper}. 

We cannot observe a pattern in dependence of $B$ if $B\not =0$. However, we can observe that there is a clear difference when using mutation rates $2/n$ or $3/n$ in the (1+1) EA for our linear function with the uniform constraint. A mutation rate of $2/n$ performs best in our experiments which we attribute to the fact that mutations flipping $2$ bits are crucial to optimize the objective function. However, also the mutation rate of $3/n$ shows a better runtime behaviour than the mutation rate of $1/n$ although it leads to a slower optimization process than the mutation rate of $2/n$. A possible explanation of the better results for mutation rate $3/n$ over $1/n$ is that 2-bit flips (or multiple bit flips) are essential to optimize the considered constraint problems when having obtained solutions at the constraint boundary.

Figure~\ref{fig:exp2} shows our experimental results for $n=500$ and $n=1000$. The results are consistent with our observations for $n=100, 200$. Again, we do not observe a dependence on the constraint bound value of $B$, $B \not = 0$. However, we can again observe that the optimization process is fastest for mutation rate $2/n$ followed by $3/n$ with both of them having a clear advantage of the standard mutation rate $1/n$.

\begin{figure}
\includegraphics[trim=0cm 2cm 0cm 0cm, clip=true, scale=0.4]{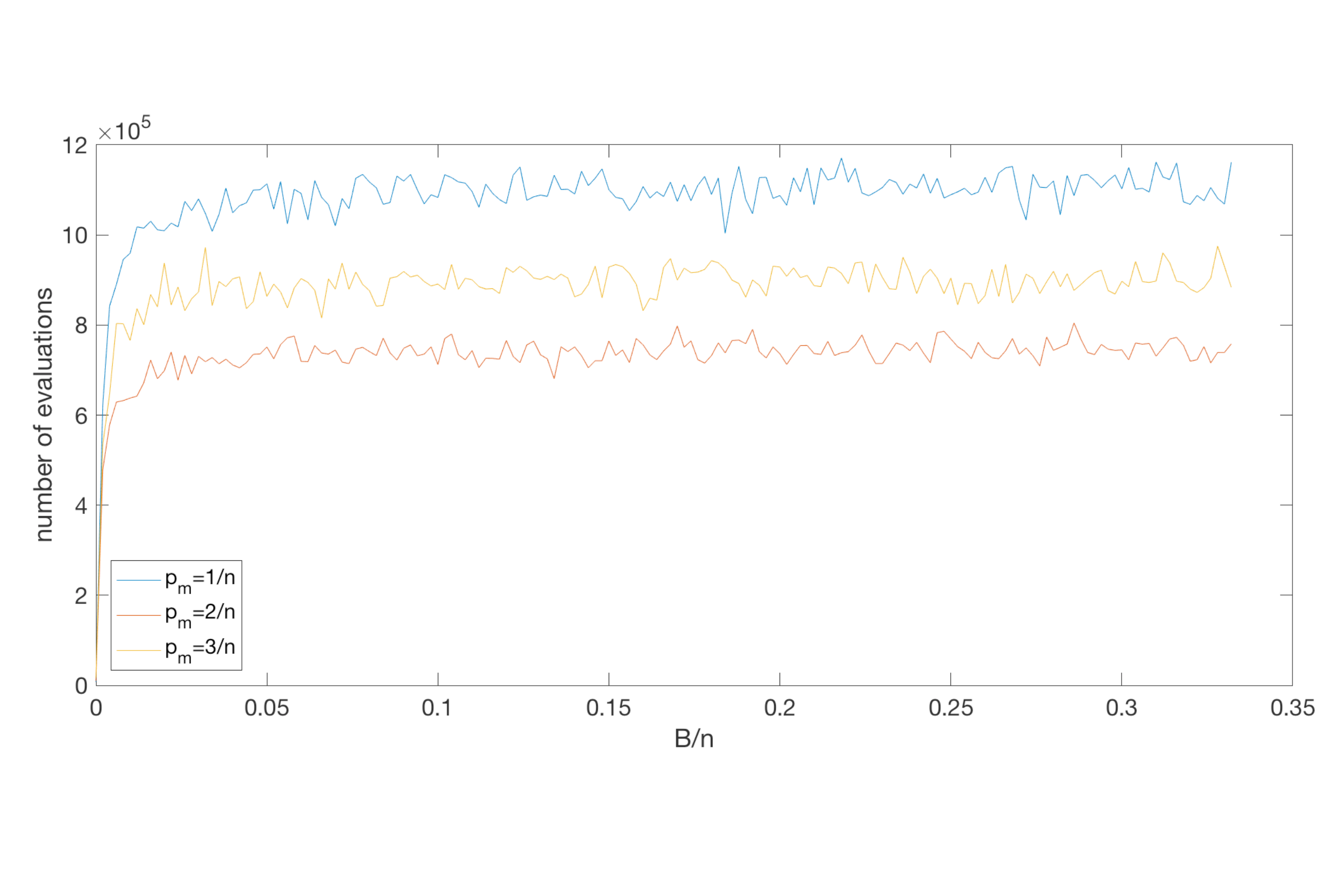}  
\includegraphics[trim=0cm 2cm 0cm 2.0cm, clip=true, scale=0.4]{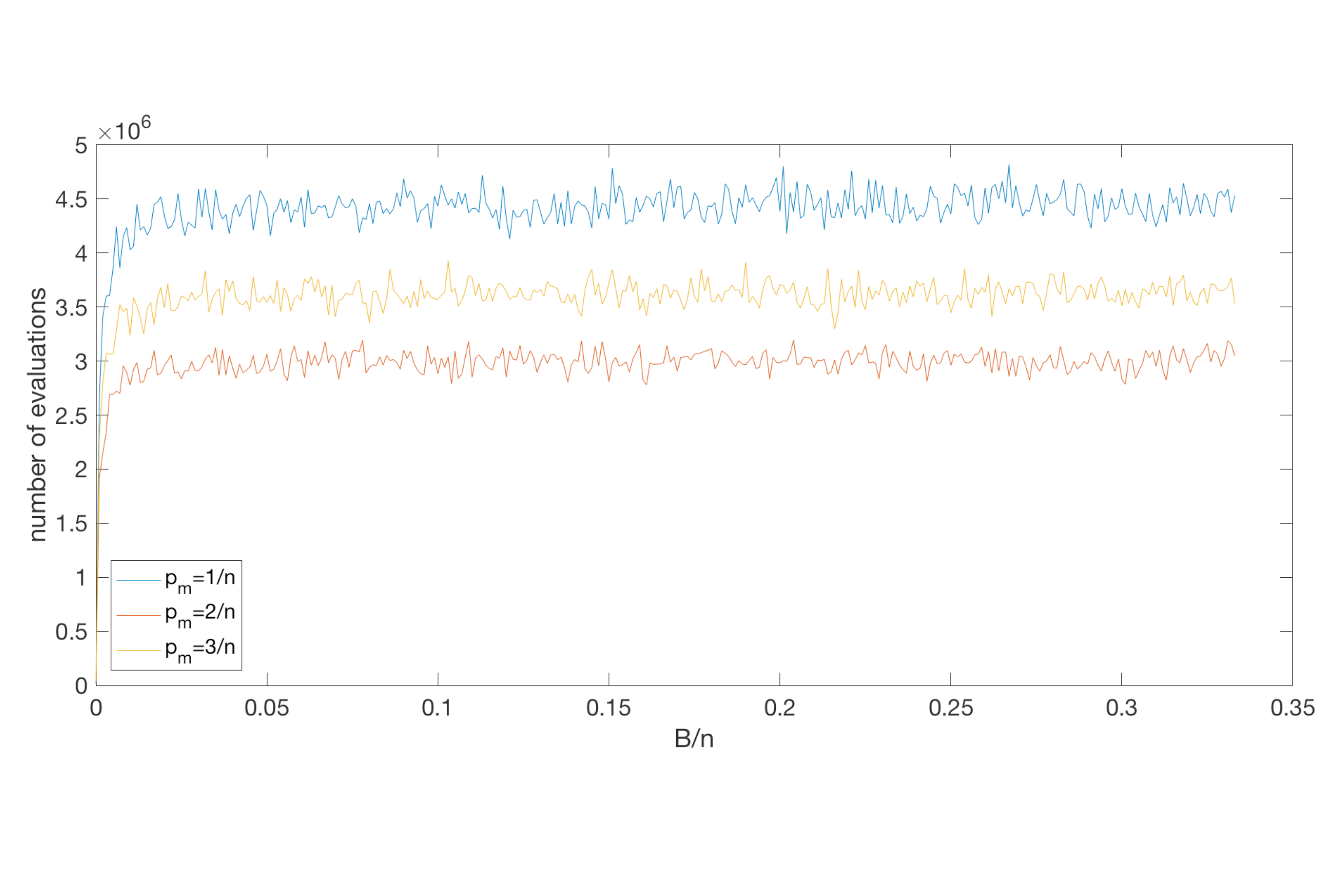}

\caption{Experimental results for $n=500, 1000$ (top and bottom) and $B=0, 1, \ldots n/3$ using (1+1) EA with mutation rate $p_m= 1/n, 2/n, 3/n$.}
\label{fig:exp2}
\end{figure}

\section{Conclusion}
\label{sec:concl}
We have carried out a rigorous theoretical analysis on the expected 
optimisation time of RLS and the \ea on the problem of minimising a linear function 
under uniform constraint. Our results include a tight expected bound of $O(n^2)$ for 
 RLS and a bound of $O(n^2\log^+B)$ for the \ea,  where $B$ is 
the constraint value, \ie, the minimum number of $1$-bits that a solution should have 
to be considered feasible. Both bounds considerably improve over previous results.

We have also proved an upper bound of $O(n^2 \log n)$ for the \ea 
with high probability and a tight bound of $O(n^2)$ in the case that the $B$ smallest weights of 
the function are identical. 
In order to prove our results for the \ea, we have conducted an adaptive 
drift analysis with a potential function that depends on the weights of the linear function 
and the constraint value~$B$. We are optimistic that the developed techniques can be helpful 
in finding upper bounds on the expected optimisation time of the \ea on more complicated problems 
for which currently best upper bounds depend on the weights of the given input. This includes the 
minimum spanning tree problem where the best proven upper bound 
for general graphs is $O(n^2 (\log n+\log w_{\max})))$ and conjectured to be $O(n^2 \log n)$.

\section*{Acknowledgement}
This research has been supported by the Australian Research Council (ARC) through grant DP160102401.

\bibliographystyle{spbasic}

\bibliography{linear-func-constraints}

\end{document}